\newif\ifpaper

\paperfalse 

\ifpaper
	\documentclass{frontiersSCNS}
	
	\usepackage{url,hyperref,microtype}
	\usepackage{lineno}
	\usepackage[onehalfspacing]{setspace}
	\linenumbers
\else
	\documentclass{article}
	\usepackage[final,nonatbib]{nips_2016_tr}
\fi

\usepackage{hyperref}       
\usepackage{url}            
\usepackage{booktabs}       
\usepackage{amsfonts}       
\usepackage{nicefrac}       
\usepackage{microtype}      

\usepackage{amsthm}
\usepackage{color}
\usepackage{amsmath}
\usepackage{mathtools}
\usepackage{graphicx}
\usepackage{multicol,array}
\usepackage{comment}
\usepackage{acronym}
\usepackage{hyperref}
\usepackage{tabulary}
\usepackage{array}
\usepackage{algpseudocode}
\usepackage{algorithmicx}
\usepackage{standalone}
\usepackage{mathrsfs}
\usepackage{tikz}
\usepackage{bbm}
\usepackage{framed}
\usepackage{rotating}
\usepackage{amssymb}
\usetikzlibrary{decorations,arrows,shapes}
\usepackage{color}

\ifpaper
\else
	\usepackage{subfigure}
\fi

\newtheorem{lemma}{Lemma}
\newtheorem{theorem}{Theorem}

\newtheorem{definition}{Definition}

\newtheorem{proposition}{Proposition}

\acrodef{HMM}{hidden Markov model}
\acrodef{KFM}{Kalman filter model}
\acrodef{CPD}{conditional probability distribution}
\acrodef{DBN}{dynamic Bayesian network}
\acrodef{DAG}{directed acyclic graph}
\acrodef{BN}{Bayesian network}
\acrodef{MDL}{minimum description length}
\acrodef{PDF}{probability density function}
\acrodef{BIC}{Bayesian information criterion}
\acrodef{AIC}{Akaike information criterion}
\acrodef{GDS}{graph dynamical system}
\acrodef{CPT}{conditional probability table}
\acrodef{EM}{expectation-maximisation}
\acrodef{ODE}{ordinary differential equation}

\ifpaper
	\setcitestyle{aysep={}}
	\renewcommand\cite{\citep}
\else
	\newcommand{\citep}[1]{\cite{#1}}
\fi

\ifpaper 
	\def\keyFont{\fontsize{8}{11}\helveticabold }
	\def\firstAuthorLast{Cliff {et~al.}} 
	\def\Authors{Oliver M. Cliff\,$^{1,*}$, Mikhail Prokopenko\,$^2$ and Robert Fitch\,$^{1,3}$}
	
	\title[An Information Criterion for Distributed Dynamical Systems]{An Information Criterion for Inferring Coupling of Distributed Dynamical Systems}
\else
	\title{An Information Criterion for Inferring Coupling of Distributed Dynamical Systems}
	\author{Oliver M.~Cliff$^*$ \And Mikhail Prokopenko$^\dagger$\And Robert Fitch$^{*,\ddagger}$\\
			  \AND \\[-5mm]
			  $^*$Australian Centre for Field Robotics,\\
			  The University of Sydney, Australia,\\
			  \texttt{o.cliff@acfr.usyd.edu.au}
			  \And \\[-5mm]
			  $^\dagger$Complex Systems Research Group,\\
			  The University of Sydney, Australia,\\
			  \texttt{mikhail.prokopenko@sydney.edu.au}
			  \AND \\[-5mm]
			  $^\ddagger$Centre for Autonomous Systems,\\
			  University of Technology Sydney, Australia,\\
			  \texttt{robert.fitch@uts.edu.au}}
\fi

\begin{document}

\maketitle

\begin{abstract}
The behaviour of many real-world phenomena can be modelled by nonlinear dynamical systems whereby a latent system state is observed through a filter. We are interested in interacting subsystems of this form, which we model by a set of coupled maps as a synchronous update \acl{GDS}. Specifically, we study the structure learning problem for spatially distributed dynamical systems coupled via a \acl{DAG}. Unlike established structure learning procedures that find locally maximum posterior probabilities of a network structure containing latent variables, our work exploits the properties of dynamical systems to compute globally optimal approximations of these distributions. We arrive at this result by the use of time delay embedding theorems. Taking an information-theoretic perspective, we show that the log-likelihood has an intuitive interpretation in terms of information transfer.

\end{abstract}
\section{Introduction}

\acresetall

Complex systems are broadly defined as systems that comprise interacting nonlinear components~\cite{boccaletti06a}. Discrete-time complex systems can be represented using graphical models such as \emph{\acp{GDS}}~\cite{wu05a,mortveit07a}, where spatially distributed dynamical units are coupled via a directed graph. The task of learning the structure of such a system is to infer directed relationships between variables; in the case of dynamical systems, these variables are typically hidden~\cite{kantz04a}. In this paper, we study the \emph{structure learning} problem for complex networks of nonlinear dynamical systems coupled via a \ac{DAG}. Specifically, we formulate synchronous update \acp{GDS} as \acp{DBN} and study this problem from the perspective of information theory.

The structure learning problem for distributed dynamical systems is a precursor to inference in systems that are not fully observable. This case encompasses many practical problems of known artificial, biological and chemical systems, such as neural networks~\cite{vicente11a,lizier11a,schumacher15a}, multi-agent systems~\cite{gan14a,xu13a,cliff16a,umenberger16a} and various others~\cite{boccaletti06a}. Modelling a partially observable system as a dynamical network presents a challenge in synthesising these models and capturing their global properties~\cite{boccaletti06a}. In addressing this challenge, we draw on probabilistic graphical models (specifically \ac{BN} structure learning) and nonlinear time series analysis (differential topology).

In this paper we exploit the properties of discrete-time multivariate dynamical systems in inferring coupling between latent variables in a \ac{DAG}. Specifically, the main focus of this paper is to analytically derive a measure (score) for evaluating the fitness of a candidate \ac{DAG}, given data. We assume the data are generated by a certain family of multivariate dynamical system and are thus able to overcome the issue of latent variables faced by established structure learning algorithms. That is, under certain assumptions of the dynamical system, we are able to employ time delay embedding theorems~\cite{stark03a,deyle11a} to compute our scores.

Our main result is a tractable form of the log-likelihood function for synchronous \acp{GDS}. Using this result, we are able to directly compute the \emph{\ac{BIC}}~\cite{schwarz78a} and \emph{\ac{AIC}}~\cite{akaike74a} and thus achieve globally optimal approximations of the posterior distribution of the graph.  Finally, we show that the log-likelihood and log-likelihood ratio can be expressed in terms of \emph{collective transfer entropy}~\cite{lizier10b,vicente11a}. This result places our work in the context of effective network analysis~\cite{sporns04a,park13a} based on information transfer~\cite{honey07a,lizier11a,cliff13a,cliff16a} and relates to the information processing intrinsic to distributed computation~\cite{lizier08b}.

\section{Related Work}

We are interested in classes of systems whereby dynamical units are coupled via a graph structure. These types of systems have been studied under several names, including complex dynamical networks~\cite{boccaletti06a}, spatially distributed dynamical systems~\cite{kantz04a,schumacher15a}, master-slave configurations (or systems with a skew product structure)~\cite{kocarev96a}, and coupled maps~\cite{kaneko92a}. Common to each of these definitions is that the multivariate state of the system comprises individual subsystem states, the dynamics of which are given by a set of either discrete-time maps or first-order \acp{ODE}, called a flow. We assume the discrete-time formulation, where a map can be obtained numerically by integrating differential equations or recording experimental data (observations) at discrete-time intervals~\cite{kantz04a}. The literature on coupled dynamical systems is often focused on the analysis of characteristics such as stability and synchrony of the system. In this work we draw on the fields of \ac{BN} structure learning and nonlinear time series analysis to infer coupling between spatially distributed dynamical systems.

\ac{BN} structure learning comprises two subproblems: \emph{evaluating} the fitness of a graph, and \emph{identifying} the optimal graph given this fitness criterion~\cite{chickering02a}. The evaluation problem is particularly challenging in the case of graph dynamical systems, which include both latent and observed variables. A number of theoretically optimal techniques exist for the evaluation problem for \acp{BN} with complete data~\cite{lam94a,bouckaert94a,heckerman95b}, which have been extended to \acp{DBN}~\cite{friedman98a}. With incomplete data, however, the common approach is to resort to approximations that find local optima, e.g., \ac{EM}~\cite{friedman98a,ghahramani98a}. An additional caveat with respect to structure learning is that algorithms find an equivalence class of networks with the same Markov structure, and not a unique solution~\cite{chickering02a}.


In nonlinear time series analysis, the problem of inferring coupling strength and causality in complex systems has received significant attention recently~\cite{schreiber00a,hoyer09a}. Early work by Granger defined causality in terms of the predictability of one system linearly coupled to another~\cite{granger69a}. Although this measure is popular for identifying coupling, it requires systems are linear statistical models and is considered insufficient for inferring coupling between dynamical systems due to inseparability~\cite{sugihara12a}. Another method popular in neuroscience is transfer entropy, which was introduced to quantify the information transfer between nonlinear (finite-order Markov) systems~\cite{schreiber00a}. Transfer entropy has been used to recover interaction networks in numerous fields such as multi-agent systems~\cite{cliff16a} and effective networks in neuroscience~\cite{vicente11a,lizier11a,lizier12c}. More recently, researchers have used the additive noise model~\cite{hoyer09a,peters11a} to infer unidirectional cause and effect relationships with observed random variables and find a unique \ac{DAG} (as opposed to an equivalence class). These studies have been extended by exploring weakly additive noise models for learning the structure of systems of observed variables with nonlinear coupling~\cite{gretton09a}.

A recent approach to inferring causality is convergent cross-mapping (CCM), which is based on Takens theorem~\cite{takens81a} and tests for causation (predictability) by considering the history of observed data of a hidden variable in predicting the outcome of another~\cite{sugihara12a}. Using a similar approach, \ifpaper\else Schumacher et al.~\cite{schumacher15a} \fi used Stark's bundle delay embedding theorem~\cite{stark99a,stark03a} to predict one subsystem from another using Gaussian processes. This algorithm can thus be used to infer the driving systems in spatially distributed dynamical systems in a similar manner to our work. However, both papers do not consider the problem of inference over the entire network structure, or formally derive the measures used therein. In our work, we provide a rigorous proof based on established structure learning procedures and discuss the problem of inference within a distributed dynamical system.


\section{Background}

This section summarises relevant technical concepts used throughout the paper. First, a stochastic temporal process $X$ is defined as a sequence of random variables $(X_{1},X_2,\ldots,X_N)$ with a realisation $(x_1,x_2,\ldots,x_N)$ for countable time indices $n\in\mathbb{N}$. Consider a collection of $M$ processes, and denote the $i$th process $X^i$ to have associated realisation $x^i_n$ at temporal index $n$, and $\boldsymbol{x}_n$ as all realisations at that index $\boldsymbol{x}_n=\langle x_n^1,x_n^2,\ldots,x_n^M\rangle$. If $X^i_n$ is a discrete random variable, the number of values the variable can take on is denoted $|X^i_n|$. The following sections collect results from \ac{DBN} literature, attractor reconstruction, and information theory that are relevant to this work.

\subsection{Dynamic Bayesian networks (DBNs)}

\acp{DBN} are a general graphical representation of a temporal model, representing a probability distribution over infinite trajectories of random variables $(\boldsymbol{Z}_1,\boldsymbol{Z}_2,\ldots)$ compactly. These models are a more expressive framework than the hidden Markov Model and Kalman filter model (or linear dynamical system)~\cite{friedman98a}. In this work, we denote $\boldsymbol{Z}_n=\{\boldsymbol{X}_n,\boldsymbol{Y}_n\}$ as the set of hidden and observed variables, respectively, where $n \in \{1,2,\ldots\}$ is the temporal index.

\acp{BN} $B = (G, \Theta)$ represent a joint distribution $p( \boldsymbol{z} )$ graphically and consist of: a \ac{DAG} $G$ and a set of \ac{CPD} parameters $\Theta$. \acp{DBN} $B = ( B_1, B_\rightarrow )$ extend the \ac{BN} to model temporal processes and comprise two parts: the prior \ac{BN} $B_1=(G_1,\Theta_1)$, which defines the joint distribution $p_{B_1}(\boldsymbol{z}_1)$; and the \emph{two-time-slice Bayesian network (2TBN)} $B_{\rightarrow} = (G_\rightarrow,\Theta_\rightarrow)$, which defines a first-order Markov process $p_{B_\rightarrow}( \boldsymbol{z}_{n+1} \mid \boldsymbol{z}_{n} )$~\cite{friedman98a}. This formulation allows for a variable to be conditioned on its respective parent set $\Pi_{G_\rightarrow}(Z^i_{n+1})$ that can come from the preceding time slice or the current time slice, as long as $G_\rightarrow$ forms a \ac{DAG}. The 2TBN probability distribution factorises according to $G_{\rightarrow}$ with a local \ac{CPD} $p_D$ estimated from an observed dataset. That is, given a set of stochastic processes $( \boldsymbol{Z}_1, \boldsymbol{Z}_2, \ldots, \boldsymbol{Z}_N )$, the realisation of which constitutes a dataset $D=(\boldsymbol{z}_1,\boldsymbol{z}_2,\ldots,\boldsymbol{z}_N)$, we obtain the 2TBN distribution as
 \begin{equation} \label{eq:bn}
  p_{B_{\rightarrow}}( \boldsymbol{z}_{n+1} \mid \boldsymbol{z}_{n} ) = \prod_{i} p_{B_\rightarrow}( z^i_{n+1} \mid \pi_{G_\rightarrow}(Z^i_{n+1}) ),
 \end{equation}
where $\pi_\mathcal{G_\rightarrow}(Z^i_{n+1})$ denotes the (index-ordered) set of realisations $\{ z^{j}_{o} : Z^j_o \in \Pi_\mathcal{G_\rightarrow}(Z^i_{n+1}) \}$. 

\subsection{Embedding theory}

Embedding theory refers to methods from differential topology for inferring the (hidden) state of a dynamical system from a reconstructed sequence of observations. The state of a discrete-time dynamical system is given by a point $\boldsymbol{x}_n$ confined to a $d$-dimensional manifold $\mathcal{M}$. The time evolution of this state is described by a map $f:\mathcal{M} \to \mathcal{M}$, so that the sequence of states $( \boldsymbol{x}_n )$ is given by $\boldsymbol{x}_{n+1} = f( \boldsymbol{x}_n )$. In many situations we only have access to a filtered, scalar representation of the state, i.e., the measurement $y_n = \psi( \boldsymbol{x}_n )$ given by some \emph{measurement function} $\psi : \mathcal{M} \to \mathbb{R}$~\cite{takens81a,stark99a}.

The celebrated Takens' theorem~\cite{takens81a} shows that for typical $f$ and $\psi$, it is possible to reconstruct $f$ from the observed time series up to some smooth coordinate change. More precisely, fix some $\kappa$ (the \emph{embedding dimension}) and some $\tau$ (the \emph{time delay}), then define the \emph{delay embedding map} $\boldsymbol{\Phi}_{f,\psi}: \mathcal{M} \to \mathbb{R}^\kappa$ by
\begin{equation} \label{eq:delay-embedding}
	\boldsymbol{\Phi}_{f,\psi}( \boldsymbol{x}_n ) = y^{(\kappa)}_n = \langle y_{n}, y_{n-\tau}, y_{n-2\tau},\ldots, y_{n-(\kappa-1)\tau} \rangle.
\end{equation}
In differential topology, an \emph{embedding} refers to a smooth map $\boldsymbol{\Psi}: \mathcal{M} \to \mathcal{N}$ between manifolds $\mathcal{M}$ and $\mathcal{N}$ if it maps $\mathcal{M}$ diffeomorphically onto its image; therefore, $\boldsymbol{\Phi}_{f,\psi}$ has a smooth inverse $\boldsymbol{\Phi}_{f,\psi}^{-1}$. The implication of Takens' theorem is that for typical $f$ and $\psi$, the image $\boldsymbol{\Phi}_{f,\psi}(\mathcal{M})$ of $\mathcal{M}$ is completely equivalent to $\mathcal{M}$ itself, apart from the smooth invertible change of coordinates given by the mapping $\boldsymbol{\Phi}_{f,\psi}$. An important consequence of this theorem is that we can define a map $\mathbf{F} = \boldsymbol{\Phi}_{f,\psi} \circ f \circ \boldsymbol{\Phi}_{f,\psi}^{-1}$ on $\boldsymbol{\Phi}_{f,\psi}$, such that $ y^{(\kappa)}_{n+1} = \mathbf{F}( y^{(\kappa)}_n ) $~\cite{stark99a}.

There are technical assumptions for Takens' theorem (and the generalised versions employed herein) to hold. These assumptions require: $(f,\psi)$ to be generic functions (in terms of Baire space), a restricted number of periodic points, and distinct eigenvalues at each neighbourhood of these points~\cite{takens81a,stark99a,stark03a,deyle11a}.

\subsection{Information theoretic measures}
\emph{Conditional entropy} represents the uncertainty of a random variable $X$ after taking into account the outcomes of another random variable $Y$ by
\begin{equation} \label{eq:conditional-entropy}
	H( X \mid Y ) = -\sum_{ x, y } p( x, y ) \log_2 p( x \mid y ).
\end{equation}
Multivariate \emph{transfer entropy} is a measure that computes the information transfer from a set of source processes to a set of destination process~\cite{lizier11a}. In this work, we use the formulation of collective transfer entropy~\cite{lizier10b}, where the information transfer from $m$ source processes $\boldsymbol{V} = \{ Y^{1}, Y^{2}, \ldots, Y^{m} \}$ to a single destination process $Y$ can be decomposed as a sum of conditional entropy terms:
\begin{equation} \label{eq:te}
	T_{ \boldsymbol{V} \to Y} = H \left( Y_{n+1} \mid Y^{(\kappa)}_{n} \right) - H\left( Y_{n+1} \mid Y^{(\kappa)}_{n}, \langle Y^{i,(\kappa^{i})}_n \rangle \right),
\end{equation}
where $Y^{i,(\kappa^i)}_n = \langle Y^{i}_n, Y^{i}_{n-\tau^i}, Y^{i}_{n-2\tau^i}, \ldots, Y^{i}_{n-(\kappa^i-1)\tau^i} \rangle$ for some $\kappa^i$ and $\tau^i$, and similarly for $Y^{(\kappa)}_n$.

\section{Representing nonlinear dynamical networks as \acp{DBN}} \label{sec:gds-dbn}

We express multivariate dynamical systems as a synchronous update \ac{GDS} to allow for generic maps. With this model, we can express the time evolution of the \ac{GDS} as a stationary \ac{DBN}, and perform inference and learning on the subsequent graph. We formally state the network of dynamical systems as a special case of the sequential \ac{GDS}~\cite{mortveit01a} with an observation function for each vertex.
\begin{definition}[Synchronous graph dynamical system (GDS)] \label{def:graph-dynamical-system}
	A synchronous \ac{GDS} is a tuple $(G,\boldsymbol{x}_n,\boldsymbol{y}_n,\{ f^{i} \}, \{ \psi^i \})$ that consists of:
	\begin{itemize}
		\item a finite, directed graph $G=(\mathcal{V},\mathcal{E})$ with edge-set $\mathcal{E} = \{E^i\}$ and $M$ vertices comprising the vertex set $\mathcal{V}=\{V^i\}$;
		\item a multivariate state $\boldsymbol{x}_n=\langle x^i_n\rangle$, composed of states for each vertex $V^i$ confined to a $d^i$-dimensional manifold $x^i_n \in \mathcal{M}^{i}$;
		\item an $M$-variate observation $\boldsymbol{y}_n=\langle y^i_n\rangle$, composed of scalar observations for each vertex $y^i_n \in \mathbb{R}$;
		\item a set of local maps $\{ f^i \}$ of the form $f^i:\mathcal{M} \to \mathcal{M}^i$, which update synchronously and induce a global map $f:\mathcal{M} \to \mathcal{M}$; and
		\item a set of local observation functions $\{ \psi^1,\psi^2,\ldots,\psi^M \}$ of the form $\psi^i:\mathcal{M}^i\to\mathbb{R}$.
	\end{itemize} 
\end{definition}
Without loss of generality, we can use local functions to describe the time evolution of the subsystems:
\begin{gather}
	x^i_{n+1}=f^i(x^i_n, \langle x^{ij}_n \rangle_j) + \upsilon_{f^i} \label{eq:dynamics} \\
	y^i_{n+1} = \psi^i( x^i_{n+1} ) + \upsilon_{\psi^i}. \label{eq:observation}
\end{gather}
Here, $\upsilon_{f^i}$ is i.i.d. additive noise and $\upsilon_{\psi^i}$ is noise that is either i.i.d. or dependent on the state, i.e., $\upsilon_{\psi^i}( x^i_{n+1} )$. The subsystem dynamics~\eqref{eq:dynamics} are therefore a function of the subsystem state $x^i_n$ and the subsystem parents' state $\langle x^{ij}_n \rangle_j$ at the previous time index such that $f^i:\mathcal{M}^{i} \times_j \mathcal{M}^{ij} \to \mathcal{M}^i$. Each subsystem observation is given by~\eqref{eq:observation}. We assume the functions $\{f^i\}$ and $\{\psi^i\}$ are invariant w.r.t. time and thus the graph $G$ is stationary.

The time evolution of a synchronous \ac{GDS} can be modelled as a \ac{DBN}. First, each subsystem vertex $V^i = \{X^i_n, Y^i_n\}$ has an associated state variable $X^i_n$ and observation variable $Y^i_n$; the parents of subsystem $V^i$ are denoted $\Pi_G(V^i)$. Since the graph $G_\rightarrow$ is stationary and synchronous, parents of $X^i_{n+1}$ come strictly from the preceding time slice, and additionally $\Pi_{G_\rightarrow}( Y^i_{n+1} ) = X^i_{n+1}$. Thus, we can build the edge set $\mathcal{E} = \{E^1,E^2,\ldots,E^M\}$ in the \ac{GDS} by means of the \ac{DBN}. That is, each edge subset $E^i$ is built by the \ac{DBN} edges \[ E^i = \{ V^j \to V^i : X^j_n \in \Pi_{G_\rightarrow}( X^i_{n+1} ) \wedge V^j \in \mathcal{V} \setminus V^i \},\] so long as $G$ forms a \ac{DAG}.

As an example, consider the synchronous \ac{GDS} in Fig.~\ref{fig:graph}. The subsystem $V^3$ is coupled to both subsystem $V^1$ and $V^2$ through the edge set $\mathcal{E} = \{ V^1 \to V^3, V^2 \to V^3 \}$. The time-evolution of this network is shown in Fig.~\ref{fig:dbn}, where the top two rows (processes $X^1$ and $Y^1$) are associated with subsystem $V^1$, and similarly for $V^2$ and $V^3$. The distributions for the state~\eqref{eq:dynamics} and observation~\eqref{eq:observation} of $M$ arbitrary subsystems can therefore be factorised according to~\eqref{eq:bn}:
\begin{equation} \label{eq:factored-dbn}
	p_{B_{\rightarrow}}( \boldsymbol{z}_{n+1} \mid \boldsymbol{z}_{n} ) = \prod^M_{i=1} p_D( x^i_{n+1} \mid x^i_n, \langle x^{ij}_{n} \rangle_j ) \cdot p_D( y^i_{n+1} \mid x^i_{n+1} ).
\end{equation}

In the rest of the paper we use simplified notation, given this constrained graph structure. Firstly, since our focus is on learning coupling between distributed systems, the superscripts refer to individual \emph{subsystems}, not variables. Thus, although the 2TBN $B_\rightarrow$ is constrained such that $\Pi_{G_\rightarrow}(Y^i_n) = X^i_n$, the notation $Y^{ij}_n$ denotes the \emph{measurement variable} of the $j$th parent of subsystem $i$, e.g., in Fig.~\ref{fig:graphs} an arbitrary ordering of the parents gives $Y^{3,1}_n=Y^1_n$ and $Y^{3,2}_n=Y^2_n$. Secondly, the scoring functions for the 2TBN network $B_\rightarrow$ can be computed independently of the prior network $B_1$~\cite{friedman98a}. We will assume the prior network is given, and focus on learning the 2TBN. As a result, we drop the subscript and note that all references to the network $B$ are to the 2TBN. Since $B_\rightarrow$ is stationary, learning $B_\rightarrow$ is equivalent to learning the synchronous \ac{GDS}.


\ifpaper
\else
	\begin{figure}[t]
		\centering
		\subfigure[]{\raisebox{5mm}{\includegraphics[width=52mm, clip=true, trim=52mm 205mm 55mm 47mm, angle=90]{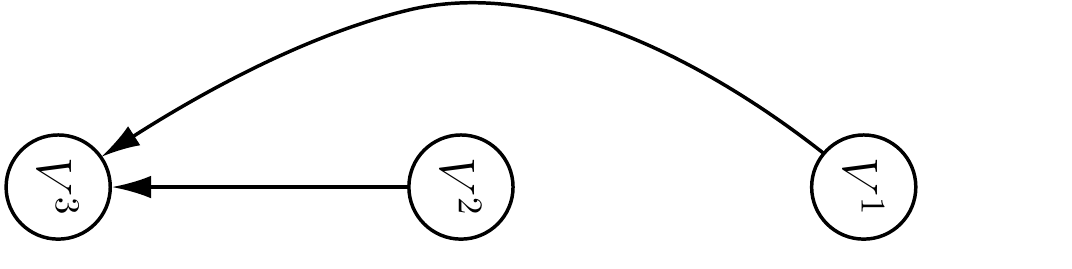}}\label{fig:graph}} \hspace{10mm}
		\subfigure[]{\includegraphics[height=55mm, clip=true, trim=50mm 148mm 10mm 43mm]{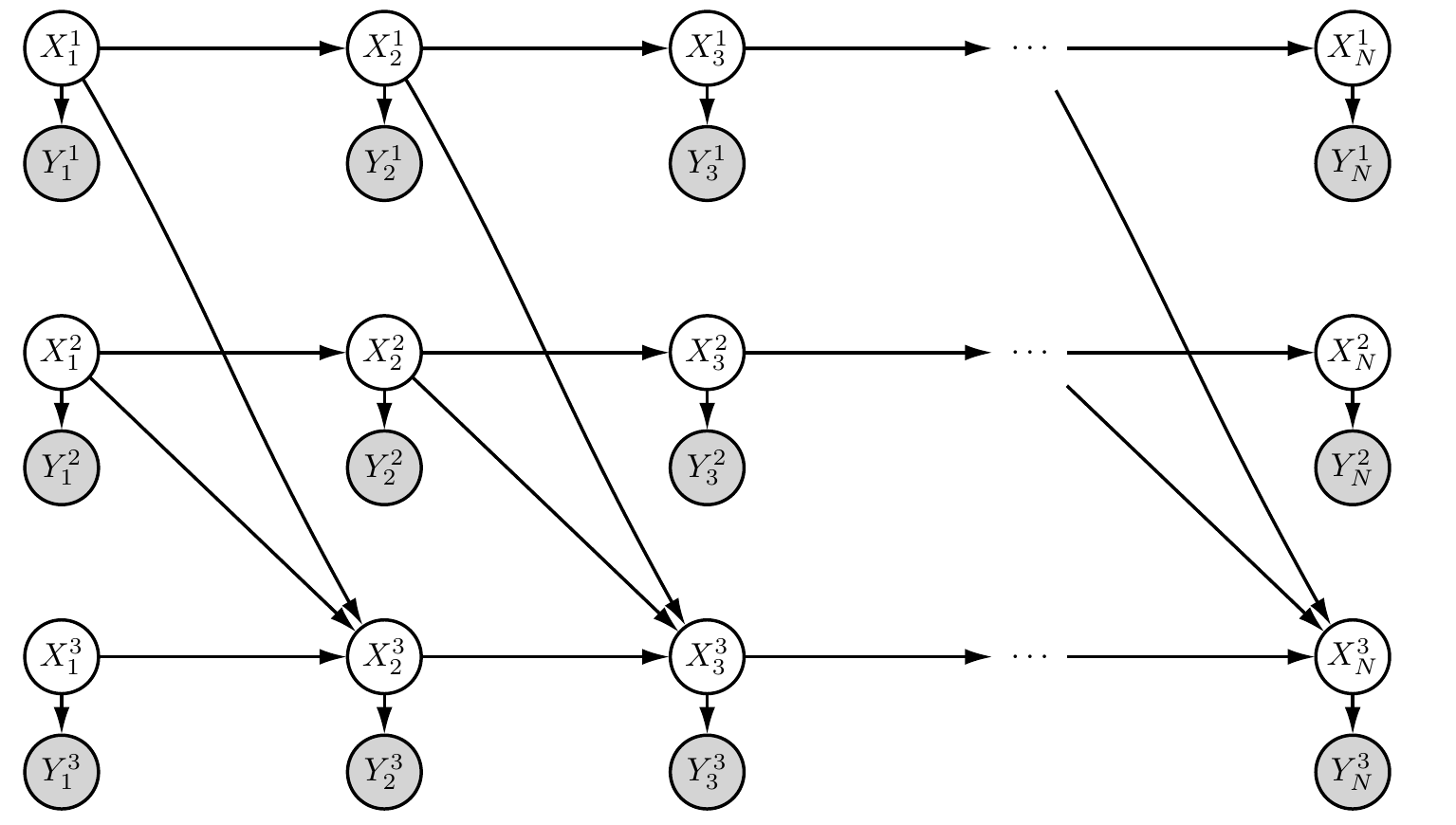}\label{fig:dbn}}
		\caption{Representation of (a)~the \ac{GDS} with three vertices ($V^1$, $V^2$ and $V^3$), and (b)~the rolled-out \ac{DBN} of the equivalent structure. Subsystem $V^3$ is coupled to both subsystems $V^1$ and $V^2$ by means of the edges between latent variables $X^1_{n} \to X^3_{n+1}$ and $X^2_{n} \to X^3_{n+1}$. 
		}\label{fig:graphs}
	\end{figure}
\fi

\section{Learning synchronous \acp{GDS} from data}


In this section we develop the theory for learning the synchronous update \ac{GDS} from data. We will focus on techniques for learning graphical models using the \emph{score and search} paradigm, the objective of which is to find a DAG $G^*$ that maximises a score $g( B : D )$. Given such a score, we can then employ established search procedures to find the optimal graph $G^*$. Thus, we can state that our main goal is to derive a tractable scoring function $g( B : D )$ for synchronous \acp{GDS} that gives a parsimonious model for describing the data. 

To derive the score we use the \ac{DBN} formulation of synchronous \acp{GDS} (Sec.~\ref{sec:gds-dbn}) to show that we cannot directly compute the posterior probability of the network structure (Sec.~\ref{sec:learning-dbns}). By making some assumptions about the system, however, we are able to compute scores for \acp{GDS} by use of attractor reconstruction methods (Sec.~\ref{sec:computing-likelihood}). We conclude this section by giving an interpretation of the log-likelihood in terms of information transfer (Sec.~\ref{sec:scoring-gds}).



\subsection{Structure learning for \acp{DBN}} \label{sec:learning-dbns}

Ideally, we want to be able to compute the posterior probability of the network structure $G$, given data $D$. Using Bayes' rule, we can express this distribution as $p( G \mid D ) \propto p( D \mid G ) p( G )$, where $p( G )$ encodes any prior assumptions we want to make about the network $G$. Thus, the problem becomes that of computing the likelihood of the data, given the model, $p( D \mid G )$. The likelihood can be written in terms of distributions over network parameters~\cite{friedman98a}:
\begin{equation} \label{eq:d-given-g}
	p( D \mid G ) = \int p( D \mid G, \Theta ) p( \Theta \mid G ) d\Theta,
\end{equation}
where we denote $\ell( \hat{\Theta}_G : D ) = \log p( D \mid G, \hat{\Theta}_G )$ as the log-likelihood function for a choice of parameters $\hat{\Theta}_G$ that maximise $p( D \mid G, \Theta )$, given a graph $G$.

A common approach to compute~\eqref{eq:d-given-g} in closed form is by using Dirichlet priors. This leads to the BD (Bayesian-Dirichlet) score and variants~\cite{heckerman95b,friedman98a}. However, to obtain this analytic solution, we require counts of the tuples $(z^i_n, \pi_G( Z^i_n ))$, which involve hidden variables. We will instead use Schwarz's~\cite{schwarz78a} asymptotic approximation of the posterior distribution, which states that
\begin{equation} \label{eq:bic}
	\lim_{N\to\infty} \log p( D \mid G ) \approx \ell( \hat{\Theta}_G : D ) - \frac{\log N}{2} C(G) + \mathcal{O}(1),
\end{equation}
where $C(G)$ is the model dimension (i.e., number of parameters needed for the graph $G$~\cite{friedman98a}) and $\mathcal{O}(1)$ is a constant bounded by the number of potential models. The approximation of the posterior~\eqref{eq:bic} requires that data come from an exponential family of likelihood functions with conjugate priors over the model $G$, and the parameters given the model $\Theta_G$~\cite{schwarz78a}.

Akaike gives a similar criterion by approximating the KL-divergence of any model from the data~\cite{akaike74a}. We can compute both criteria in terms of the log-likelihood function $\ell( \hat{\Theta}_G : D )$ and the model dimension $C(G)$, and thus the problem can be generalised to that of deriving an information criterion for scoring the graph of the form
\begin{equation} \label{eq:g-mdl}
	g(B:D) = \ell( \hat{\Theta}_G : D ) - f(N) \cdot C(G).
\end{equation}
When $f(N)=1$, we have the \ac{AIC} score~\cite{akaike74a}; $f(N)=\log(N)/2$ yields the \ac{BIC} score~\cite{schwarz78a}, and $f(N)=0$ gives the maximum likelihood score.

\subsection{Deriving the scores for synchronous \acp{GDS}} \label{sec:computing-likelihood}

To calculate the information criterion~\eqref{eq:g-mdl}, we require tractable expressions for the log-likelihood function $\ell( \hat{\Theta}_G : D )$ and the model dimension $C(G)$. The form of the \ac{CPD} in~\eqref{eq:factored-dbn} specifies these functions, and for~\eqref{eq:bic} to hold, this distribution must come from an exponential family~\cite{schwarz78a}. We do not assume the underlying model is linear-Gaussian or other known distributions, and thus express the log-likelihood as the maximum likelihood estimate for multinomial distributions~\cite{friedman98a}. From~\eqref{eq:factored-dbn} the log-likelihood then decomposes as
\begin{align} \label{eq:ll-prob}
		\ell( \hat{\Theta}_G : D ) &= -N \sum_{i=1}^M \left[ \sum_{x^i_{n+1}} \sum_{ \langle x^{ij}_{n} \rangle_j } p_D( x^i_{n+1}, x^i_{n}, \langle x^{ij}_{n} \rangle_j ) \log p_D( x^i_{n+1} \mid x^i_n, \langle x^{ij}_{n} \rangle_j ) \right. \nonumber \\
			 &\hspace{25mm} + \left. \sum_{x^i_{n+1}}\sum_{ y^i_{n+1} } p_D( y^i_{n+1}, x^i_{n+1} ) \log  p_D( y^i_{n+1} \mid x^i_{n+1} ) \right].
\end{align}
Note that although we describe the states and observations as discrete in~\eqref{eq:ll-prob}, we assume the data are generated by a continuous and stationary process. In theory it is conceivable to have access to an infinite dataset containing realisations of all potential states and observations. In practice we have a limited dataset and therefore must implement a discretisation scheme. Modelling the dynamical systems with non-parametric techniques requires that the number of parameters scales linearly in the size of the data, and thus  $C(G)$ scales linearly with $N$. Instead, later we will assume the observation data are discretised, such that there are $|Y^i_n|$ possible outcomes for an observed random variable $Y^i_n$.

The log-likelihood function~\eqref{eq:ll-prob} involves distributions over latent variables, and thus we resort to state-space (attractor) reconstruction. First, Lemma~\ref{lem:gds-observation} shows that a future observation from a given subsystem can be predicted from a sequence of past observations. Building on this result, we present a computable formulation of the 2TBN distribution $p_{B_\rightarrow}( \boldsymbol{z}_{n+1} \mid \boldsymbol{z}_{n} )$ via Lemma~\ref{lem:graph-cpd}. We then derive a tractable form of the log-likelihood function, presented in Lemma~\ref{lem:ll}. It is then shown in Theorem~\ref{th:max-ll} that these lemmas allow us to compute the information criterion~\eqref{eq:g-mdl}.

\begin{lemma} \label{lem:gds-observation}
	Consider a synchronous \ac{GDS} $(G,\boldsymbol{x}_n,\boldsymbol{y}_n,\{ f^{i} \}, \{ \psi^i \})$, where the graph $G$ is a \ac{DAG}. Each subsystem state follows the dynamics $x^i_{n+1} = f^i( x^i_n, \langle x^{ij}_n \rangle_j )$ and emits an observation $y^i_{n+1} = \psi^i( x^i_{n+1} )$; the subsystem observation can be estimated, for some map $\mathbf{G}^i$, by
	\begin{equation} \label{eq:g}
		y^i_{n+1} = \mathbf{G}^i\left( y^{i,(\kappa^i)}_n, \langle y^{ij,(\kappa^{ij})}_n \rangle_j \right).
	\end{equation}
\end{lemma}
\begin{proof}
	Consider a forced system $\boldsymbol{x}_{n+1} = f( \boldsymbol{x}_n, w_n )$ with forcing dynamics $\boldsymbol{w}_{n+1} = h( \boldsymbol{w}_n )$ and observation $y_n = \psi( \boldsymbol{x}_{n+1} )$. Given this type of forced system, the bundle delay embedding theorem \cite{stark99a,stark03a} states that the delay map $\boldsymbol{\Phi}_{f,h,\psi}( \boldsymbol{x}_n, \boldsymbol{w}_n ) = y^{(\kappa)}_n$  is an embedding for generic $f$, $\psi$, and $h$. Stark~\cite{stark99a} proved this result in the case of forcing dynamics $h$ that are independent of the state $x$.\footnote{Stark~\cite{stark99a} conjectures that the theorem should generalise to functions $h$ that are not independent of $x$. To the best of our knowledge, this result remains to be proven.} For notational simplicity, we omit dependence on the noise process for the map $\boldsymbol{\Phi}_{f,h,\psi}$; the noise can be considered an additional forcing system so long as $\upsilon_f$ is i.i.d and $\upsilon_\psi$ is either i.i.d or dependent on the state~\cite{stark03a}.

	Given a \ac{DAG} $G$, any ancestor of the subsystem $V^i$ is not dependent on $V^i$. As such, the sequence
	\begin{equation} \label{eq:phi}
		y^{i,(\kappa^i)}_n = \boldsymbol{\Phi}_{f^i,\langle f^{ij} \rangle_j, \psi^i} \left( x^i_n,\langle x^{ij}_n \rangle_j \right)
	\end{equation}
	is an embedding, since $\langle x^{ij}_n \rangle_j$ is independent of $x^i_n$. Let $\langle  x^{ijk}_n \rangle_k$ be the index ordered set of parents of node $X^{ij}_n$ (which itself is the $j$th parent of the node $X^i_n$). Under the constraint that $G$ is a \ac{DAG}, where the state $x^i_{n+1} = f^i( x^i_n, \langle x^{ij}_n \rangle_j ) + \upsilon_{f^i}$, it follows from the bundle delay embedding theorem~\cite{stark99a,stark03a} that there exists a map $\mathbf{F}^i$ that is well defined and a diffeomorphism between observation sequences. From~\eqref{eq:phi} we can write this map
	\begin{align} \label{eq:F}
		y^{i,(\kappa^i)}_{n+1} &= \boldsymbol{\Phi}_{ f^i,\langle f^{ij} \rangle_j, \psi^i} \left( f^i \left( x^i_n,\langle x^{ij}_n \rangle_j \right), \left\langle f^{ij} \left( x^{ij}_n,\langle x^{ijk}_n \rangle_k \right) \right\rangle_j \right) \nonumber \\
					  &= \boldsymbol{\Phi}_{f^i,\langle f^{ij} \rangle_j, \psi^i} \left( f^i \left( \boldsymbol{\Phi}^{-1}_{f^i,\langle f^{ij} \rangle_j, \psi^i}( y^{i,(\kappa^i)}_n ), \left\langle \boldsymbol{\Phi}^{-1}_{f^{ij},\langle f^{ijk} \rangle_k, \psi^{ij}}( y^{ij,(\kappa^{ij})}_n ) \right\rangle_j \right) \right).
	\end{align}
	Denote the RHS of~\eqref{eq:F} as $\mathbf{F}^i ( y^{i,(\kappa^i)}_n, \langle y^{ij,(\kappa^{ij})}_n \rangle_j )$; the last $\kappa^i+\sum_j \kappa^{ij}$ components of $\mathbf{F}^i$ are trivial. Denote the first component as $\mathbf{G}^i:\mathbb{R}^{\kappa^i}\times_j\mathbb{R}^{\kappa^{ij}} \to \mathbb{R}$, then we arrive at~\eqref{eq:g}.
\end{proof}

\begin{lemma} \label{lem:graph-cpd}
	Given an observed dataset $D=(\boldsymbol{y}_1,\boldsymbol{y}_2,\ldots,\boldsymbol{y}_N)$ where $\boldsymbol{y}_n\in\mathbb{R}^M$ are generated by a directed and acyclic synchronous \ac{GDS} $(G, \boldsymbol{x}_n,\boldsymbol{y}_n, \{ f^i \}, \{ \psi^i \})$, the 2TBN distribution can be written as
	\begin{equation} \label{eq:cpd2}
		\prod^M_{i=1} p_D( x^i_{n+1} \mid x^i_n, \langle x^{ij}_{n} \rangle_j ) \cdot p_D( y^i_{n+1} \mid x^i_{n+1} ) = \frac{ \prod_{i=1}^M p_D( y^i_{n+1} \mid y^{i,(\kappa^i)}_n,\langle y^{ij,(\kappa^{ij})}_n \rangle_j ) }{ p_D( \boldsymbol{x}_n \mid \langle y^{i,(\kappa^i)}_n \rangle ) }.
	\end{equation}
\end{lemma}
\begin{proof}
The generalised time delay embedding theorem~\cite{deyle11a} states that, under certain technical assumptions, and given $M$ inhomogeneous observation functions $\{ \psi^1, \psi^2, \ldots, \psi^M \}$, the map
\begin{equation} \label{eq:phis}
	\boldsymbol{\Phi}_{f,\psi}(\boldsymbol{x}) = \langle \boldsymbol{\Phi}_{f^1,\psi^1}(\boldsymbol{x}),\boldsymbol{\Phi}_{f^2,\psi^2}(\boldsymbol{x}), \ldots,\boldsymbol{\Phi}_{f^M,\psi^M}(\boldsymbol{x}) \rangle
\end{equation}
is an embedding where each subsystem (local) map $\boldsymbol{\Phi}_{f^i,\psi^i}:\mathcal{M}\to\mathbb{R}^{\kappa^i}$, and, at time index $n$ is described by
\begin{equation} \nonumber
	\boldsymbol{\Phi}_{f^i,\psi^i}(\boldsymbol{x}_n) = y^{i,(\kappa^i)}_n = \langle \psi^i \left( \boldsymbol{x}_n \right), \psi^{i} ( \boldsymbol{x}_{n-\tau^i} ), \psi^{i} ( \boldsymbol{x}_{n-2\tau^i} ), \ldots, \psi^i ( \boldsymbol{x}_{n-(\kappa^i-1)\tau^i} ) \rangle
\end{equation}
where $\sum_i \kappa^i=2d+1$~\cite{deyle11a}.\footnote{The original proof~\cite{deyle11a} uses positive lags, however the authors note that the use of negative lags also applies (and should be used in the case of endomorphisms~\cite{takens02a}).} Therefore, the global map~\eqref{eq:phis} is given by $\boldsymbol{\Phi}_{f,\psi}(\boldsymbol{x}_n) = \langle y^{i,(\kappa^i)}_n \rangle$ and there must exist an inverse map $\boldsymbol{x}_n = \boldsymbol{\Phi}_{f,\psi}^{-1}\left( \langle y^{i,(\kappa^{i})}_n \rangle\right)$. Given Lemma~\ref{lem:gds-observation}, the existence of $\boldsymbol{\Phi}_{f,\psi}^{-1}$, and since $\forall i, \{y^{i,(\kappa^i)}_n, \langle y^{ij,(\kappa^{ij})}_n \rangle_j\} \subseteq \langle y^{i,(\kappa^{i})}_n \rangle$, we arrive at the following equation:
\begin{align} \label{eq:end2}
		\prod_{i=1}^M p_D\left( Y^i_{n+1} = \mathbf{G}^i\left( y^{i,(\kappa^i)}_n, \langle y^{ij,(\kappa^{ij})}_n \rangle_j \right) \mid y^{i,(\kappa^i)}_n, \langle y^{ij,(\kappa^{ij})}_n \rangle_j \right) & \nonumber \\
		&\hspace{-75mm}= p_D\left( \boldsymbol{X}_n = \boldsymbol{\Phi}_{f,\psi}^{-1}\left( \langle y^{i,(\kappa^{i})}_n \rangle \right) \mid \langle y^{i,(\kappa^{i})}_n \rangle \right)  \\
		&\hspace{-67mm} \times \prod_{i=1}^M p_D\left( X^i_{n+1} = f^i( x^i_n, \langle x^{ij}_n \rangle_j ) \mid x^i_n, \langle x^{ij}_n \rangle_j \right) \cdot p_D\left( Y^i_{n+1} = \psi^i( x^i_{n+1} ) \mid x^i_{n+1} \right). \nonumber
\end{align}
Rearranging~\eqref{eq:end2} gives the equality in~\eqref{eq:cpd2}.
\end{proof}
Lemma~\ref{lem:graph-cpd} shows that the distributions can be reformulated by conditioning on delay vectors. The RHS of~\eqref{eq:cpd2} can be used to perform inference in the 2TBN~\eqref{eq:factored-dbn}. The numerator is a product of local \acp{CPD} of scalar variables, and can thus be computed by either counting (for discrete variables) or density estimation (for continuous variables). The denominator is used to compute the probability that the hidden state occured, given an observed delay vector; fortunately, \ifpaper\else Casdagli~\cite{casdagli91a} \fi established methods to compute this \ac{CPD} for a variety of practical scenarios. Therefore, Lemma~\ref{lem:graph-cpd} provides a method to perform exact inference. Using this delay vector representation, we arrive at the following theorem.
\begin{theorem} \label{lem:ll}
Consider a synchronous \ac{GDS} $(G,\boldsymbol{x}_n,\boldsymbol{y}_n,\{f^{i}\}, \{\psi^i\})$, where the graph $G$ is a \ac{DAG}. Each subsystem state follows the dynamics $x^i_{n+1} = f^i( x^i_n, \langle x^{ij}_n \rangle_j )$ and generates an observation $y^i_{n+1} = \psi^i( x^i_{n+1} )$; a complete dataset is given by the sequence of observations $D = (\boldsymbol{y}_1, \boldsymbol{y}_2, \ldots, \boldsymbol{y}_N)$. The log-likelihood of the data given a network structure can be computed in terms of conditional entropy:
\begin{align} \label{eq:log-ll-cond-entropy}
	\ell( \hat{\Theta}_{G} : D ) &= N \cdot H( \boldsymbol{X}_n \mid \langle Y^{i,(\kappa^i)}_n \rangle ) - N \cdot \sum_{i=1}^M H( Y^i_{n+1} \mid Y^{i,(\kappa^i)}_{n}, \langle Y^{ij,(\kappa^{ij})}_n \rangle_j )
\end{align}
\end{theorem}
\begin{proof}
	Substituting~\eqref{eq:cpd2} into~\eqref{eq:ll-prob} gives the log-likelihood $\ell( \hat{\Theta}_{G} : D )$ as
	\begin{align} \label{eq:ll3}
			&\hspace{-10mm} N \sum_{i=1}^M \sum_{y^i_{n+1}} \sum_{y^{i,(\kappa^{i})}_n} \sum_{\langle y^{ij,(\kappa^{ij})}_n \rangle_j } p_D( y^i_{n+1}, y^{i,(\kappa^{i})}_n, \langle y^{ij,(\kappa^{ij})}_n \rangle_j ) \log p_D( y^i_{n+1} \mid y^{i,(\kappa^{i})}_n, \langle y^{ij,(\kappa^{ij})}_n \rangle_j ) \nonumber \\
	 &\hspace{-5mm} - N \sum_{\boldsymbol{x}_n} \sum_{\langle y^{i,(\kappa^{i})}_n \rangle } p_D( \boldsymbol{x}_n, \langle y^{i,(\kappa^i)}_n \rangle ) \log p_D( \boldsymbol{x}_n \mid \langle y^{i,(\kappa^i)}_n \rangle ).
	\end{align}
	In~\eqref{eq:ll3} we have removed arguments of the joint distributions that will be nullified when multiplied with the \ac{CPD}. Expressing~\eqref{eq:ll3} in terms of conditional entropy~\eqref{eq:conditional-entropy}, we arrive at~\eqref{eq:log-ll-cond-entropy}.
\end{proof}


\begin{theorem} \label{th:max-ll}
The information criterion~\eqref{eq:g-mdl} for synchronous {GDS} can be computed as:
\begin{align} \label{eq:g-bic2}
	g( B : D ) &= - N \cdot \sum_{i=1}^M H( Y^i_{n+1} \mid Y^{i,(\kappa^i)}_{n}, \langle Y^{ij,(\kappa^{ij})}_n \rangle_j ) \nonumber \\
							  &\hspace{10mm} - f(N) \cdot \sum_{i=1}^M \Big( | Y^i_n |^{\kappa^i} \cdot (|Y^i_n| - 1 ) \cdot \prod_{ V^p \in \Pi_G( V^i ) } | Y^p_n |^{\kappa^{p}} \Big).
\end{align}
\end{theorem}
\begin{proof}
The distributions for the first term in~\eqref{eq:ll3} do not depend on the parents of a subsystem and thus are independent of the graph $G$ being considered. Therefore, we have the following equation for maximimum log-likelihood:
\begin{equation} \label{eq:max-log-ll}
	\max_{G} \ell( \hat{\Theta}_G : D ) = \mathcal{O}(N) - N \cdot \min_G \sum_{i=1}^M H( Y^i_{n+1} \mid Y^{i,(\kappa^i)}_{n}, \langle Y^{ij,(\kappa^{ij})}_n \rangle_j ).
\end{equation}
We can now compute the number of parameters needed to specify the model as~\cite{friedman98a}
\begin{equation} \label{eq:dim-g2}
	C(G) = \sum_{i=1}^M \Big( | Y^i_n |^{\kappa^i} \cdot \left(|Y^i_n| - 1 \right) \cdot \prod_{ V^p \in \Pi_G( V^i ) } | Y^{p}_n |^{\kappa^{p}} \Big).
\end{equation}
Since we are searching for the graph $G^* = \max_G{ g( B : D ) }$, holding $N$ constant, we can substitute~\eqref{eq:max-log-ll} and~\eqref{eq:dim-g2} into~\eqref{eq:g-mdl} and ignore the constant term $\mathcal{O}(N)$ in~\eqref{eq:max-log-ll}.
\end{proof}

\subsection{The log-likelihood and information transfer} \label{sec:scoring-gds}
To conclude our study of the scores, we look at the log-likelihood in the context of information transfer. First, rearranging the terms of collective transfer entropy~\eqref{eq:te} we can rewrite the log-likelihood function~\eqref{eq:log-ll-cond-entropy}, leading to the following result.
\begin{proposition} \label{prop:ll-decomp}
The log-likelihood function for the synchronous \ac{GDS}~\eqref{eq:log-ll-cond-entropy} decomposes as follows:
\begin{equation} \label{eq:ll}
	\ell( \hat{\Theta}_{G} : D ) = N \cdot H( \boldsymbol{X}_n \mid \langle Y^{i,(\kappa^i)}_n \rangle ) - N \cdot \sum_{i=1}^M H( Y^i_{n+1} \mid Y^{i,(\kappa^i)}_n ) + N \cdot \sum_{i=1}^M  T_{ \langle Y^{ij} \rangle_j \to Y^i }.
\end{equation}
\end{proposition}
\vspace{-2ex}
Again, the first two terms in~\eqref{eq:ll} do not depend on the proposed graph structure, and thus maximising log-likelihood is equivalent to maximising collective transfer entropy. This becomes clear when we consider the \emph{log-likelihood ratio}. This ratio quantifies the gain in likelihood by modelling the data $D$ by a candidate network $B$ instead of the empty network $B_\emptyset$, i.e., $$\ell( \hat{\Theta}_{G} : D ) - \ell( \hat{\Theta}_{G_\emptyset} : D ) \propto \log \frac{ p( B \mid D ) }{ p( B_\emptyset \mid D ) }.$$
Recall that the empty \ac{DAG} $G_\emptyset$ is one with no parents for all vertices $\forall i, \Pi_G( V^i ) = \langle Y^{ij,(\kappa^{ij})}_n \rangle_j = \emptyset$. Substituting this definition into~\eqref{eq:log-ll-cond-entropy} (or, alternatively~\eqref{eq:ll}) gives the following result.
\begin{proposition} \label{prop:ll-ratio}
The ratio of the log-likelihood~\eqref{eq:log-ll-cond-entropy} of a candidate \ac{DAG} $G$ to the empty network $G_\emptyset$ can be expressed as $$\ell( \hat{\Theta}_{G} : D ) - \ell( \hat{\Theta}_{G_\emptyset} : D ) = N \cdot \sum_{i=1}^M T_{\langle Y^{ij} \rangle_j \to Y^i}.$$
\end{proposition}

\section{Discussion and future work}
We have presented a principled method to score the structure of nonlinear dynamical networks, where dynamical units are coupled via a \ac{DAG}. We approached the problem by modelling the time evolution of a synchronous \ac{GDS} as a \ac{DBN}. We then derived the \ac{AIC} and \ac{BIC} scoring functions for the \ac{DBN} based on time delay embedding theorems. Finally, we have shown that the log-likelihood of the synchronous \ac{GDS} can be interpreted in the context of information transfer.

The representation of synchronous \acp{GDS} as \acp{DBN} allows for inference of coupling in dynamical networks and facilitates techniques for synthesis in these systems. \acp{DBN} are an expressive framework that allow representation of generic systems, as well as a numerous general purpose inference techniques that can be used for filtering, prediction, and smoothing~\cite{friedman98a}. Our representation therefore allows for probabilistic reasoning for purposes of planning and prediction in complex systems.

Theorem~\ref{th:max-ll} captures an interesting parallel between learning from complete data and learning nonlinear dynamical networks. If the embedding dimension $\kappa$ and time delay $\tau$ are unity, then the information criterion becomes identical to learning a \ac{DBN} from complete data~\cite{friedman98a}. Thus, our result could be considered a generalisation of typical structure learning procedures.

The results presented here provoke new insights into the concepts of structure learning, nonlinear time series analysis and effective network analysis~\cite{sporns04a,park13a} based on information transfer~\cite{honey07a,lizier11a,cliff13a,cliff16a}. The information-theoretic interpretation of the log-likelihood has interesting consequences in the context of information dynamics and information thermodynamics of nonlinear dynamical networks. The transfer entropy terms in Propositions~\ref{prop:ll-decomp} and~\ref{prop:ll-ratio} show that the optimal structure of a synchronous \ac{GDS} is immediately related to the information processing of distributed computation~\cite{lizier08b}, as well as the thermodynamic costs of information transfer~\cite{prokopenko14a}.


In the future, we aim to perform empirical studies to exemplify the properties of the presented scoring functions. Specifically, the empirical studies should yield insight into the effect of weak, moderate and strong coupling between dynamical units. An important concept to consider in stochastic systems is the convergence of the shadow (reconstructed) manifold to the true manifold~\cite{sugihara12a}; we have implicitly accounted for this phenomena by using \acp{CPD} in our model, however it is important to investigate the property of convergence with different density estimation techniques. In addition, we are interested in the effect of synchrony in these networks and the relationship to previous results for dynamical systems coupled by spanning trees~\cite{wu05a}. We conjecture that approach used here will allow us to derive scoring functions without the assumption of multinomial observations, and thus afford the use of non-parametric density estimators. Parametric techniques, such as learning the parameters of dynamical systems~\cite{ghahramani99a,hefny15a}, could be considered in place of the posterior approximations.

Finally, the reconstruction theorems used in this paper typically make the assumption that the map (or flow) is a diffeomorphism (invertible in time). Thus, given any state, the past and future are uniquely determined and the time delay $\tau$ can be taken positive or negative. In certain cases, however, the time-reversed system is acausal, giving a map that is not time-invertible (an endomorphism). Ideally, we would aim to have methods to infer coupling for both endomorphisms and diffeomorphisms. \ifpaper\else Takens~\cite{takens02a} \fi showed that if the map is an endomorphism, taking the delay vector of temporally \emph{previous} observations forms an embedding. The generalised theorems in~\cite{stark99a,stark03a,deyle11a}, however, were established for diffeomorphisms, rather than endomorphisms; we can only conjecture that taking a delay of past observations (as we have done throughout this paper) follows for these results. Empirical studies using the measures presented in this paper would indicate whether it is an important line of inquiry to prove the generalised reconstruction theorems for endomorphisms.

\section*{Acknowledgements}
We would like to thank Joseph Lizier, J\"{u}rgen Jost, and Wolfram Martens for many helpful discussions, particularly in regards to embedding theory.

This work was supported in part by the Australian Centre for Field Robotics; the New South Wales Government; and the Faculty of Engineering \& Information Technologies, The University of Sydney, under the Faculty Research Cluster Program.

\ifpaper
\else
	\bibliographystyle{ieeetr}
\fi
\bibliography{./structure-learning}

\end{document}